\documentclass[times, 10pt,twocolumn, a4paper]{article}
\usepackage{latex12}
\usepackage{times}
\usepackage{balance}
\usepackage{algorithm}
\RequirePackage[algo2e,lined]{algorithm2e}
\usepackage{framed,multirow}
\usepackage{paralist}
\usepackage{amssymb}
\usepackage{caption}
\usepackage{graphicx}
\usepackage{amsmath}
\usepackage{comment}
\usepackage{latexsym}
\let\appendix\artappendix
\newcommand{\BlackBox}{\rule{1.5ex}{1.5ex}}
\newenvironment{proof}{\par\noindent{\bfseries\upshape
  Proof\ }}{\hfill\BlackBox\\[2mm]}

\newtheorem{theorem}{Theorem}

\newtheorem{proposition}[theorem]{Proposition}

\newtheorem{definition}[theorem]{Definition}

\usepackage{url}
\usepackage{xcolor}
\definecolor{newcolor}{rgb}{.8,.349,.1}
\usepackage{color}

\begin{document}

\title{Combinatorial optimization for low bit-width neural networks}

\author{
Han Zhou \thanks{Corresponding author: han.zhou@esat.kuleuven.be. We acknowledge the Research Foundation - Flanders (FWO) through project number G0A1319N, and the Flemish Government (AI Research Program). HZ is supported by the China Scholarship Council.} $^1$~~~
Aida Ashrafi $^2$~~~
Matthew B. Blaschko $^1$~~~
\smallskip 
\\
$^1$ \textit{Dept. ESAT, Center for Processing Speech and Images,  KU Leuven, Belgium}
\\
$^2$ \emph{University of Bergen, Norway}
}

\maketitle
\thispagestyle{empty}

\begin{abstract}
Low-bit width neural networks have been extensively explored for deployment on edge devices to reduce computational resources. Existing approaches have focused on gradient-based optimization in a two-stage train-and-compress setting or as a combined optimization where gradients are quantized during training. Such schemes require high-performance hardware during the training phase and usually store an equivalent number of full-precision weights apart from the quantized weights. In this paper, we explore methods of direct combinatorial optimization in the problem of risk minimization with binary weights, which can be made equivalent to a non-monotone submodular maximization under certain conditions. We employ an approximation algorithm for the cases with single and multilayer neural networks. For linear models, it has $\mathcal{O}(nd)$ time complexity where $n$ is the sample size and $d$ is the data dimension. We show that a combination of greedy coordinate descent and this novel approach can attain competitive accuracy on binary classification tasks.
\end{abstract}

\Section{Introduction}
\label{sec1}
 The emergence of binarized neural networks (BNNs) has made it possible to train DNNs for cost- and power-limited devices such as cars, smartphones, tablets, TVs, etc. Especially when both activations and network weights are binary, the multiply-accumulate operations in BNNs can be implemented via XNOR-popcount operations~\cite{rastegariECCV16}, which greatly reduces the cost of BNN inference. Methods proposed to date either strongly rely on naive binarization over pre-trained full precision models via stochastic gradient descent (SGD)~\cite{courbariaux2015binaryconnect} or make improvements in forward binarization and backward propagation~\cite{cai2017deep,ding2019regularizing}. However, there is still a fundamental gap in the literature in making discrete optimization applicable to BNNs. 

Current methods are extremely efficient at test time, but the training procedure still requires high computing resources.  For low power mobile, embedded, or medical devices, access to GPU resources may be limited or absent.  In order for such devices to learn local models, low resource training procedures must also be developed. Fully combinatorial optimization inherently uses low-cost integer/bit operations that can be optimized for deployment on resource limited hardware.

In this paper, we focus on a fully combinatorial optimization for the models where weights are binary or low-bit, and analyze the problem from the perspective of submodular analysis, a discrete analogue of convex analysis that leads to efficient optimization algorithms with global approximation guarantees. Specifically, we find that under certain conditions, risk minimization with discretized weights is equivalent to a non-monotone supermodular minimization problem.
Our contributions are summarized as follows:
\begin{inparaenum}[(i)]
    \item We present two learning strategies: greedy coordinate descent (GCD) and randomized supermodular minimization (RSM) for low-bit width single-layer neural networks on binary classification tasks. These can be further extended to multi-layer neural networks under certain circumstances. 
    \item Our experiments on the MNIST and CIFAR10 datasets show that both strategies have their pros and cons. RSM comes with approximation guarantees and is more computationally efficient than GCD, but empirically the latter can obtain higher accuracy in some settings. A hybrid of these two approaches provides a controllable trade-off between model performance and time complexity.
\end{inparaenum}

\subsection{Related Work}
\noindent \textbf{Submodular Maximization:} Combinatorial optimization is the process of searching for the maxima (or minima) of an objective function whose domain is a discrete space. Among these discrete optimization scenarios, submodular functions are attractive for a wide variety of problems because of their diminishing return property, which yields polynomial time algorithms with optimal approximation rates. In contrast with monotone functions, non-monotone functions fit well to the underlying mechanics of real-world systems. For the maximization of non-monotone submodular functions, \cite{feige2011maximizing} has proved that for the unconstrained setting, it is NP-hard to obtain better than a $\frac{1}{2}$-approximation in the value oracle sense. \cite{6375344} provides a randomized algorithm that achieves exactly this rate. \cite{gillenwater2014maximization} summarizes the recent developments in the maximization of non-monotone submodular functions. Similar to our work, \cite{9186137} has applied submodular optimization with cardinality constraints to feature selection on regression tasks.

\noindent
\textbf{Low-bit width neural networks:} Substantial research efforts have been devoted in recent years to optimizing BNNs as measured by latency and energy-efficiency \cite{Chang}. \cite{qin2020binary} has summarized binarization methods into two categories: \begin{inparaenum}[(i)] \item naive BNNs, which are based on quantization over the weights and activations, \item optimization-based approaches, e.g.\ minimizing the quantization error, improving the network loss function, and reducing the gradient error\end{inparaenum}. Additionally, there have been several attempts to ternarize the weights e.g.\ $\{+\alpha, 0 , -\alpha\}$ \cite{lin2015neural,li2016ternary}.

Our work operates in the same solution space of binary, ternary, and higher fixed bit width weights, but approaches the problem from the perspective of combinatorial optimization for the first time.

\section{Problem Statement} \label{section: network_model} 
We begin by presenting the network model and summarizing notation used in this paper. For the binary classification task, we consider the following bias-free two-layer fully connected feedforward neural network $f:\mathbb{R}^{d} \rightarrow \mathbb{R}$ with d-dimensional input $x\in \mathbb{R}^{d}$ and $h$ hidden units 
\begin{equation}\small
    f(x) = a \sigma(W^Tx)= \sum_{j=1}^{h} a_j \sigma(\langle  w_j, x \rangle) \label{eq:2_layer_NN}
\end{equation}
where $w_j \in \{\alpha, \beta\}^{d}$ is the weights corresponding to neuron $j$, $a_j \in \mathbb{R}$ is the $j$-th coefficient of the second layer, and $\sigma(\cdot)$ is the ReLU function. 
\begin{figure}[htb]
\vspace{-0.2cm}
\centering
\includegraphics[width=0.35\textwidth]{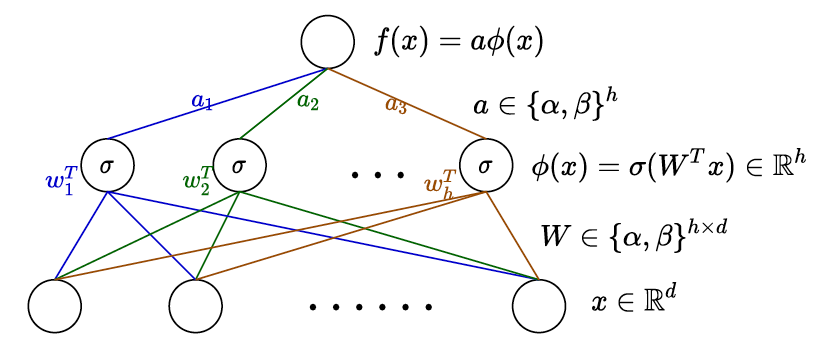}
\caption{Bias-free two-layer network}
\label{fig:2LayerNetwork}
\vspace{-0.5cm}
\end{figure}
For concise notation, we write $a = [a_1, \cdots, a_h] \in \mathbb{R}^{h}$ for the coefficients of the second layer, $W = [w_1, \cdots, w_h] \in \{\alpha, \beta\}^{h\times d}$ for the weight matrix, $X = [ x_1, \cdots, x_n]$ for the data matrix, $y\in \{-1,+1\}^{n}$ for the labels.

The approximation of the desired prediction function $f: \mathcal{X} \to \mathcal{Y}$ $(\mathbb{R}^{d} \to \mathbb{R})$ is made based on a training set of $n$ observations $D_n = \{(x_i,y_i)\}_{i =1}^n$ that are assumed to satisfy the independent and identically distributed (i.i.d.) sampling assumption \cite{vapnik2013nature}. The weights of this model are learned by minimizing the empirical risk of the available data:
\begin{equation}\small \label{eq:empiricalRisk} 
    \mathcal{R}_{emp}(f):=\frac{1}{n} \sum_{i=1}^{n} \ell \left(y_{i}, f\left(x_{i}\right)\right)
\end{equation}

where $\ell(y,f(x))$ is the loss incurred by outputting $f(x)$ when $y$ is the true label. Typically, this loss is most often convex in its second argument i.e. squared error or logistic loss, etc, and it can be viewed as a set function where we map from a set indicating which elements of $w$ are equal to $\alpha$, to the empirical risk of the weights represented by that set. This is a key observation that allows us to analyze the problem from the perspective of submodular analysis. We are interested in the underlying structure of this empirical risk such that it can be solved efficiently via combinatorial algorithms. 
\section{Method}
We begin this section with a brief introduction to submodularity and convexity, and then give our main theorems and outline the general arguments. We first consider a single-layer neural network and introduce the theorems and algorithms for learning this model. 

\subsection{Submodularity and Convexity}
Submodular functions play an important role in combinatorial optimization when minimizing (or maximizing) a set function defined on the power set $\mathcal{P}(V)$, similar to convex functions on continuous spaces.
\begin{definition}[Submodularity] \label{DefinitionSubmodularity}
A set function $f: \mathcal{P}(V) \rightarrow \mathbb{R}$ is submodular if and only if for all $B \subseteq A \subset V$ and $x \in V \backslash B,$ it holds
\begin{align}\small
f(B \cup\{x\})-f(B) \geq f(A \cup\{x\})-f(A) .
\end{align}
\end{definition}
A function $f$ is supermodular if $-f$ is submodular.
\begin{definition}[Convexity] \label{DefinitionConvexity}
A function $g : \mathbb{R}^p \rightarrow \mathbb{R}$ is convex if for all $x_1,x_2 \in \mathbb{R}^p$
\begin{align*}\small
    g\left(\frac{1}{2} (x_1+x_2) \right) \leq \frac{1}{2} \left( g\left(x_1\right) + g\left(x_2\right) \right)  .
\end{align*}
\end{definition}

\subsection{Learning Linear Functions}
 We can express the output of the single-layer network as $f(x) = \langle w, x\rangle$, where $w$ is in the discrete domain $\{\alpha, \beta \}^{d}$ (without loss of generality, we may assume that $\alpha< \beta$). Let $\Omega=\{1,\cdots,d\}$ and $S=\{i: w_i =\beta \}$ indicate the entries of $w$ that equal $\beta$, the loss function $L(S):2^{\Omega}\rightarrow \mathbb{R}$
\begin{equation}\small
\begin{aligned}
L(S) =\sum_{i=1}^n \log\left(1+\exp\left(-y_i \langle w, x_i\rangle\right)\right)=:g(w) \label{eq:logistic_reg} 
\end{aligned} 
\end{equation}
is a set fuction of $S$ where $\ell(\cdot)$ is taken as logistic loss. Finding $\min_{S \subseteq \Omega} L(S)$ for this set function is a NP-hard optimization problem since it generalizes max-cut.

\subsubsection{Greedy Coordinate Descent}
 Inspired by \cite{khuller1999budgeted}, we present a naive greedy coordinate descent (GCD) procedure to learn single-layer binary neural networks via coordinate descent (Algorithm ~\ref{algForwardGreedy}, supplementary material). Essentially, it starts with a weight vector $w^0= \{\alpha,\beta\}^d$ corresponding to a set $S_0$, and greedily makes change to its coordinate one by one by comparing their feed-forward evaluation loss on the training set, and stops after iterating through all the coordinates of $w$. Noticing that the change in $g(w)$ upon changing $w[i]$ from $\alpha$ to $\beta$ is actually is incurred by $\langle \Delta w, x\rangle = (\beta-\alpha)x[i]=\mathcal{O}(1)$ where $w[i]$ represents the $i$th entry of $w$. Therefore, this algorithm can be employed in the minimum time complexity $\mathcal{O}(nd)$ even in the absence of any underlying structure of the objective function, yet it does not necessarily provide any theoretical guarantees.

\subsubsection{Supermodular Minimization}

\begin{proposition} \label{prop:1}
Given an input vector $x\in \mathbb{R} ^d$, if $g : \mathbb{R} \rightarrow \mathbb{R}$ is convex and $x \geq \mathbf{0}$ or $x \leq \mathbf{0}$ element wise, then $g(\langle w, x\rangle)$ is supermodular for $w \in \{0,1\}^d$. 
\end{proposition}
\vspace{-0.1cm}
\begin{proof}
We have $s(A) = \langle w, x\rangle$ with a set $A=\{i: w_i\neq 0\}$. Hence, if $x \geq \mathbf{0}$, we have $s(A)\in\mathbb{R}_{+}$ for all set $A$. From  \cite[Proposition~6.1]{bach2011learning}, if $s\in\mathbb{R}_{+}^{p}$, and $g : \mathbb{R}_{+} \rightarrow \mathbb{R}$ is a concave function, then $F:A\mapsto g(s(A))$ is submodular.  Hence, for $x \geq \mathbf{0}$, $F:A \mapsto -g(s(A)) $ is submodular, thus the set function $g(\langle w, x\rangle)$ is supermodular. In the case of $x \leq \mathbf{0}$, take $s(A) = -\langle w, x\rangle \in \mathbb{R}_{+}$ and using the first result, we have that the set function $g(\langle w, x\rangle)$ is supermodular.
\end{proof}
\noindent Moreover, one can show that if $x\in \mathbb{R}^d$ is mixed (its entries contain both positive and negative values), then $g(\langle w, x\rangle)$ is neither submodular nor supermodular in $w$. The following proposition shows that, one can find a convex upper bound of the linear classifier $f(x) = \langle w, x\rangle $ that is supermodular in its weight vector $w \in \{0,1\}^d$. For a vector $x\in \mathbb{R}^d$, we define
$ \operatorname{pos}(x) := \max(0,x)$ where the $\max$ is applied element wise.
We also define $\operatorname{neg}(x) = - \operatorname{pos}(-x).$
\vspace{-0.1cm}
\begin{proposition} \label{prop:mixedX}
Consider a simple linear classifier $f(x) = \langle w, x\rangle $ with a convex surrogate $\ell$ that is an upper bound on the 0-1 loss: $\ell(y\langle w, x\rangle) \geq [y\langle w, x\rangle \leq 0]$ ($[\cdot]$ being Iverson bracket notation), then
\begin{align*}\small
    \frac{1}{2}\left( \ell(y\langle w, \operatorname{pos}(x)\rangle) + \ell(y\langle w,\operatorname{neg}(x)\rangle) \right) \geq [y\langle w, x\rangle \leq 0] 
\end{align*}
and the left hand side is supermodular in $w$.
\end{proposition}
\begin{proof}
Since $\ell$ is convex, then by the definition of convexity (Def.~\ref{DefinitionConvexity}), we have
\begin{equation*}\small
\begin{aligned}
 &\frac{\ell \left(y \langle w,\operatorname{pos}(x)\rangle\right) + \ell\left(y\langle w,\operatorname{neg}\left(x\right)\rangle \right) }{2} \\ & \geq \ell \left( \frac{y \langle w,\operatorname{pos}(x)\rangle+ y \langle w,\operatorname{neg}(x)\rangle }{2} \right)= \ell\left(\frac{y\langle w,x\rangle}{2}\right) \\ \label{eq:pos_neg}
 &\geq \left[ \frac{y\langle w,x\rangle}{2} \leq 0\right]= \left[ y\langle w,x\rangle \leq 0\right],
\end{aligned}
\end{equation*}
where the definition of convexity implies the first inequality. Then we show the left hand side is supermodular in $w$. Given $\ell(\cdot)$ is a convex with $\operatorname{pos}(x)\geq 0$ and $\operatorname{neg}(x)\leq 0$, both $\ell(y\langle w, \operatorname{pos}(x)\rangle)$ and $\ell(y\langle w, \operatorname{neg}(x)\rangle)$ are supermodular in $w$ (because of Prop.~\ref{prop:1}). The left hand side is a positive linear combinations of supermodular functions, hence it is supermodular in $w$ by closedness properties \cite[Section~2.1]{bach2011learning}.
\end{proof}
\vspace{-0.3cm}

\noindent We now turn to the basic conditions we impose on $w$ such that the loss function can be considered as a real-valued set function. Obviously, this is not a minimal condition: below we show that Prop.~\ref{prop:1} holds for any binary value $w\in \{\alpha,\beta\}^d$ ($\alpha \neq\beta$).

\begin{proposition} \label{prop:alphabeta}
Given $g: \mathbb{R} \rightarrow \mathbb{R}$ is a concave function and $x \geq \mathbf{0}$ or $x \leq \mathbf{0}$ element wise, the optimization for $g(\langle w, x\rangle)$ is also submodular for $w\in \{\alpha,\beta\}^d$ ($\alpha < \beta$).
\end{proposition}
\vspace{-0.1cm}
\begin{proof}
 When $w\in \{\alpha,\beta\}^d$, we can express $ g\left(\langle w, x\rangle\right)$ as a set function $g\left(s\left(B\right)\right)$ with $s(B) = \langle w, x\rangle$ related to the set $B=\{i: w_i = \beta \}$. We define a vector $w'$ that satisfies $w_i^{'} =0 $ if $i\in B$ and $w_i^{'} =1 $ if $i \neq B$ and a related set $A=\{i: w_i^{'} = 0 \}$. Thus we have $A=B$
 \begin{equation*}\small
     \begin{aligned}
         s(A) &= \langle w^{'}, x\rangle = \sum_{j\in A}[x]_j \\
         s(B) & =\langle w^{}, x\rangle = \alpha  \sum_{j\in V\backslash A}[x]_j+\beta \sum_{k\in A }[x]_k  \\
         &= \alpha \sum_{j\in V }[x]_j + \left(\beta-\alpha\right) \sum_{k\in A}[x]_k.
     \end{aligned}
 \end{equation*}
 Thus, $g\left(s\left(B\right)\right)$ can be rewritten as:
 \begin{equation*}\small
 \begin{aligned}
     g( s(B)) =g \left( \alpha \sum_{j\in V }[x]_j + \left(\beta-\alpha\right) \sum_{k\in A}[x]_k \right)
 \end{aligned}
 \end{equation*}
 where the first term $\alpha \sum_{j\in V }[x]_j$ can be viewed as a constant $c$. Thus given $g(t): \mathbb{R} \rightarrow \mathbb{R}$ is concave, the function $g  \left(c + \left(\beta-\alpha\right)t\right)$ is also a concave function of $t$. We can apply $x \geq \mathbf{0}$ to the condition in Prop.~\ref{prop:1}, hence the submodularity of $g(\langle w, x\rangle)$ for $w\in \{\alpha,\beta\}^d$.
\end{proof}
\vspace{-0.3cm}

\noindent The logistic loss \eqref{eq:logistic_reg} of the single-layer binary neural network, because of Proposition~\ref{prop:alphabeta}, is supermodular in $w\in \{\alpha,\beta\}^d$. Algorithm~\ref{algNaiveSubmodular} proposed by \cite{6375344} is adopted for minimizing this supermodular loss, which leads to exactly an expected $\frac{1}{2}$-approximation guarantee. Notice that the change in $g(\cdot)$ upon replacing the $i$th element of $w_E$ with $\beta$ is incurred by $\Delta w^Tx = (\beta-\alpha)[x]_i=\mathcal{O}(1)$, leading to a running time $\mathcal{O}(nd)$.
\begin{algorithm}[H]\small
 \SetAlgoLined
\SetKwFor{With}{with}{do}{}
\SetKwInOut{Input}{input}
\SetKwInOut{Output}{output}
\caption{Randomized Supermodular Minimization}
\Input{Dataset $D_n: \mathbb{R}^{d}\times \mathbb{R}$; binary value \{$\alpha$, $\beta$\}; supermodular function: $g(w)$}
\Output{weight vector $w_{E}^{d}(\text{or }w_{F}^{d}) \in \{\alpha,\beta\}^d$}
\BlankLine
\emph{$E \leftarrow \emptyset \text{; } F \leftarrow \Omega$; $w_{E}^0\in \{\alpha\}^d$; $w_{F}^{0}\in \{\beta\}^d$}\\
\For{$i\leftarrow 1$ \KwTo $d$}{
$w_E^{i}[i]\leftarrow \beta $ and $w_F^{i}[i]\leftarrow \alpha$ \\
$a_i = g(w_{E}^{i-1})- g(w_{E}^i)$, $b_i = g(w_{F}^{i-1}) -g(w_{F}^i)$\\
$a_i^{'} = \max(0, a_i) \text{, } b_i^{'} = \max(0, b_i)$\\
\With{probability $\frac{a^{\prime}_i}{a^{\prime}_i+b^{\prime}_i}$}{
$w_F^i \leftarrow w_F^{i-1}$
} \Other{  $w_E^i\leftarrow w_E^{i-1}$ }
}
\label{algNaiveSubmodular}
\end{algorithm}
\vspace{-0.5cm}

\subsubsection{Learning ternary and low-bit width weights}
\begin{proposition} [Ternary weights] \label{Prop:ternaryweights}
Given an input vector $x\in \mathbb{R} ^d$, if $g : \mathbb{R} \rightarrow \mathbb{R}$ is a convex function and $x \geq \mathbf{0}$ or $x \leq \mathbf{0}$ element wise, We define $w = w_1-w_2 \in \{-1,0,1\}^d$ with $w_1, w_2\in\{0,1\}^d$, then 
$$\hat{g}(\langle w, x\rangle) = \frac{g(\langle w_1, 2x\rangle) + g(\langle w_2, -2x\rangle)}{2} $$
is a convex upper bound of $g(\langle w, x\rangle)$ that is supermodular in $w_1, w_2\in\{0,1\}^d$. 
\end{proposition}
\begin{proof}
Applying Definition~\ref{DefinitionConvexity} of convexity, we have 
\begin{equation*}\small
\begin{aligned}
   \underbrace{g\left( \frac{\langle w_1, 2x\rangle - \langle w_2, 2x\rangle }{2}\right)}_{g(\langle w, x\rangle)} & \le  \underbrace{\frac{g(\langle w_1, 2x\rangle) + g(\langle w_2, -2x\rangle)}{2}}_{\hat{g}(\langle w, x\rangle)}
\end{aligned}
\end{equation*}
where Prop.~\ref{prop:1} implies the right hand side is supermodular in $w_1, w_2\in\{0,1\}^d$.
\end{proof}
\vspace{-0.5cm}

\noindent Proposition~\ref{Prop:ternaryweights} implies that one can construct a supermodular upper bound for ternary weights which can be viewed as the difference between two binary vectors $w_1, w_2\in\{\alpha,\beta\}^d$. When fixing $w_1$, the loss function~\eqref{eq:logistic_reg} is supermodular in $w_2$, and vice versa. Thus, coordinate descent can be incorporated with Algorithm~\ref{algNaiveSubmodular} to learn this weight. Moreover, we could further split $w_1$ into $w_{11}$, $w_{12} \in \{\alpha,\beta\}^d$, thus quantizing the models with more bits.

\section{Multi-Layer Networks}
In this section, we summarize the construction of a convex upper bound of the loss function for the multi-layer networks. In particular, we are interested in the case when this upper bound is supermodular such that it can be solved efficiently via various supermodular minimization algorithms.
\vspace{-0.2cm}
\subsection{Two-Layer Networks} \label{Section:2_layer_NN}
\vspace{-0.2cm}
For the bias-free two-layer fully connected feedforward neural network with a ReLU activation in Section \ref{section: network_model}, it aims to minimize the empirical risk with respect to the weight matrix $W$:
\begin{align}
    \ell(y,f(X)) = \sum_{i=1}^n \log\left(1+\exp\left(-y_if\left(x_i\right)\right)\right) \label{eq:logisticLoss}
\end{align}
where $\ell(\cdot)$ is taken as logistic loss. Based on Equation~\eqref{eq:2_layer_NN} and \eqref{eq:logisticLoss}, when fixing $a$, we can view the loss function as a set function of the entries of $W$. In particular, we are interested in the loss for observation $(x_i,y_i)$ with respect to the weights $w_k$ corresponding to the neuron $k$ while keeping other weights in the first layer being fixed:
\begin{equation}\small \label{eq:twoLayerLoss}
\begin{aligned}
    L(S_k) & 
      = \underbrace{\log\left(1+\exp\left(-y_i a_k \sigma\left( \langle  w_k, x_i \rangle\right) -  C   \right)\right) }_{=: g(\langle  w_k, x_i \rangle)} 
\end{aligned}
\end{equation}
which can be viewed as a real-value set function of a set $S_k =\{i: [w_k]_i=\beta\}$ with a constant $C = y_i \sum_{j \neq k }^{h} a_j \sigma\left(\langle  w_j, x_i \rangle\right)$. For the sake of simplicity, we assume $x_i$ is non-negative, thus the characteristic of this function can be analyzed by considering the two following two cases.
\begin{inparaenum}[(i)]\small
    \item $y_ia_k \leq 0:$  $g(t)=\log\left(1+\exp\left(p\max\left(0,t\right)\right)\right)$($p>0$) is convex in $t$. Then by directly applying Prop.~\ref{prop:alphabeta}, this loss function is supermodular for $w_k\in \{\alpha, \beta\}^{d}$.
    \item $y_ia_k > 0:$ $g(t)=\log\left(1+\exp\left(p\max\left(0,t\right)\right)\right) (p<0)$ is not convex due to the flattened part where $t<0$. As depicted in Fig.~\ref{fig:approximateLoss},  we can replace the non-convex flat portion of the loss curve with a convex upper bound $\tilde{g}(t)$: the tangent line at $t=0$ (red line) or the function without ReLU activation (blue line), which leads to supermodularity for $w_k\in \{\alpha, \beta\}^{d}$ from Prop.~\ref{prop:1}.
\end{inparaenum}

\begin{figure}[H]
\vspace{-0.2cm}
\begin{center}
\includegraphics[width=0.23\textwidth]{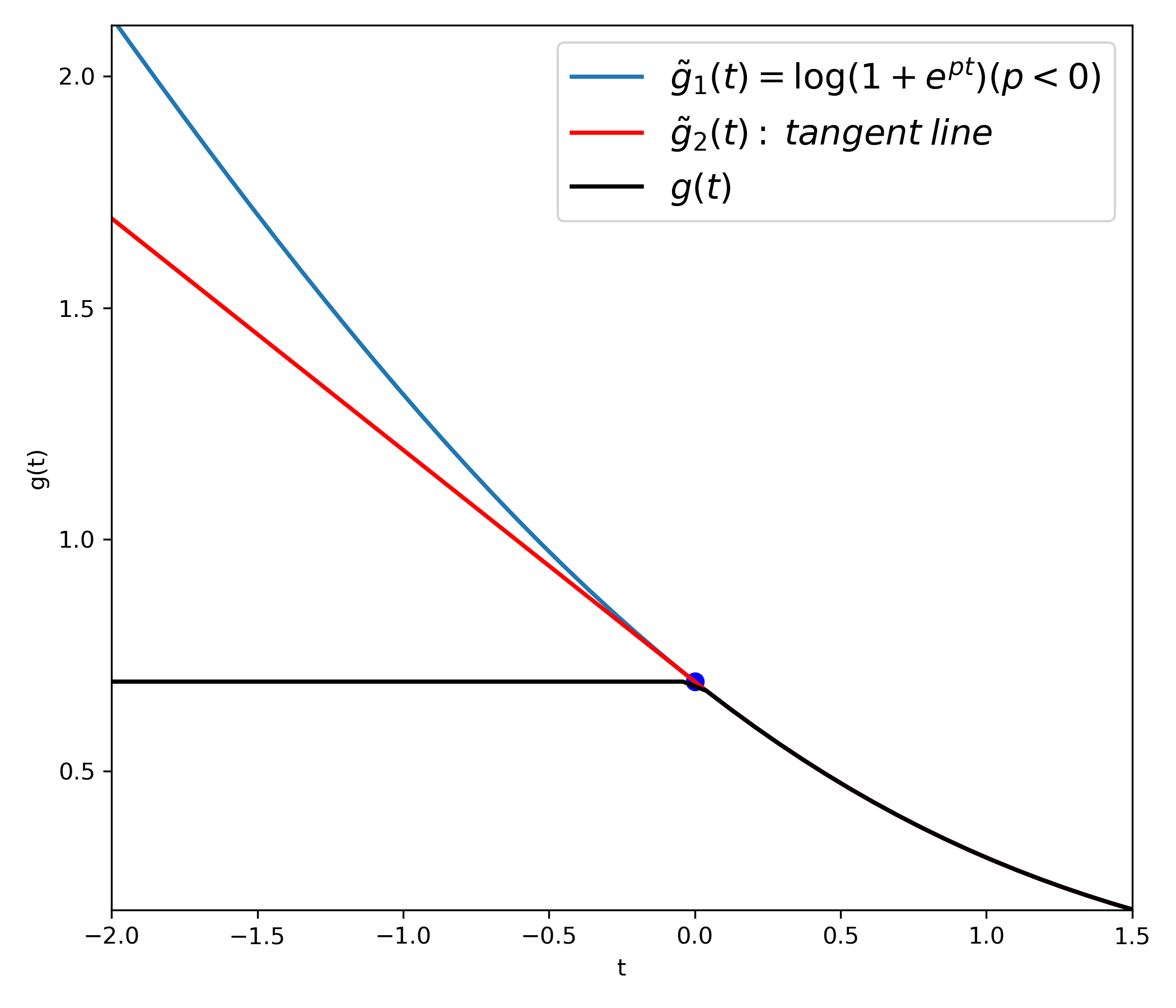}
\caption{Two convex upper bounds for $g(t)=\log\left(1+\exp\left(p\max\left(0,t\right)\right)\right) (p<0)$.}
\label{fig:approximateLoss}
\end{center}
\vspace{-0.6cm}
\end{figure}

\noindent Then the approximated empirical risk with $\tilde{g}(t)$, as a sum of supermodular functions $G(w_k) = \sum_{i=1}^{n} \tilde{g}(\langle  w_k, x_i \rangle)$, is thus supermodular in $w_k$. Coordinate descent can be incorporated with supermodular optimization in order to optimize each row of the weight matrix $W$ (Algorithm ~\ref{algo_disjdecomp}, supplementary material). Likewise, Greedy coordinate descent can be extended in the same way without imposing supermodularity on the objective function. Thus the empirical risk can be used as the objective function instead.
\vspace{-0.2cm}
\subsection{Deep Neural Networks} \label{Section:deepNeuralNetworks}
\vspace{-0.2cm}
Given the input of one conv layer $l$, we can represent a neural network with L hidden layers by a series of functions of the form:
\begin{equation}
\begin{aligned}
        \mathbf{y}_{l+1} &= \sigma(W_l^T\mathbf{y}_l+b_l) \text{ for } l\in 1\cdots L-1\\
        f(x) &= a_L\mathbf{y}_L +b_L
\end{aligned}
\end{equation}
where $\mathbf{y}_l $, the output from all previous layers to layer $l$, is a $k_{l}^2c_{l}$-by-1 vector that represents co-located $k_l\times k_l$ pixels in $c_l$ channels. With $n_l = k_{l}^2c_{l}$ denoting the number of connections of a response, $W_{l} \in \{\alpha, \beta\}^{n_{l+1}\times n_{l}}$ is a $n_{l+1}$-by-$n_{l}$ matrix, where $n_{l+1}$ is the number of filters and each row of $W_{l}$ represents the weights of a filter. The central idea of our approach is to approximate $f(x)$ via a two-layer neural network with layer $l$ as the first layer
\begin{equation}
\begin{aligned}
    \hat{f}(x) &= \hat{a} \sigma \left(W_{l}  \mathbf{y}_l+b_l \right) + \hat{b} \label{eq:linearization}
\end{aligned}
\end{equation}
where $\hat{a} = \frac{\partial f(x)}{\partial \mathbf{y}_{l+1}} $ is the gradient of $f(x)$ with respect to $\mathbf{y}_{l+1}$ at observation $x$. A bias term $\hat{b}=f(x)-\hat{a}\sigma \left(W_{l}  x_l+b_l\right)$ is estimated based on the difference between the true output and the approximate output. Then, we can use the approach in Section \ref{section: network_model} to find a convex surrogate for the empirical risk. Finally, a layerwise training strategy is applied to optimize the weights of the network (Algorithm~\ref{algoMultiLayerNetworks}, supplementary material). Given $n$ is the sample size, $L$ is the depth of the network, and $|W|$ is the number of weights in total, the time complexity for this algorithm is: $\mathcal{O}(nL^2+n|W|)$. The first term $nL^2$ measures the cost of linear approximation of the network using the gradient: linearization~\eqref{eq:linearization} is performed on all $(L-1)$ layers and all $n$ samples. Once we have the linearization (Algorithm~\ref{algo_disjdecomp}, supplementary material), the marginal gain of every entry of $W_i$ gives a cost of $n|W|$. Notice that here we ignore the cost of optimization for the last layer via stochastic average gradient. Likewise, GCD can also be extended in the same way, leading to a cost of $\mathcal{O}(nL |W|)$.

\vspace{-0.2cm}
\section{Experimental Study}
\vspace{-0.2cm}
\noindent \textbf{[Single-layer Network on MNIST]} In our first experiment, we used the single-layer network with logistic loss as the binary classifier to detect the first three digits $0-2$ among the first six digits $0-5$ of the MNIST dataset. We set $\beta=-\alpha=0.5$ for RSM.

\vspace{-5pt}
\begin{table}[H]\small
\caption{Comparison of the performance of RSM for low-bit width weights, and stochastic average gradient (SAG) for full precision weights. This series of experiments was re-run 20 times with different seeds. The loss and accuracy are measured on the testing set. The standard deviation of the loss and accuracy shows the randomness of RSM. RSM for high-bit weights, as expected, outperforms networks with low-bit weights and underperforms networks with full-precision weights.}
\centering
\begin{tabular}{@{}c|c|c@{}}
\hline
Bit-width      & Accuracy(\%) & Loss \\ \hline
1              &     $77.7\pm 4.36$          &   $0.602\pm 0.0161$   \\
2              &    $88.7\pm 1.29$                 &   $0.556\pm 5.92\times 10^{-3}$    \\
3              &      $89.1\pm 0.831$            &    $0.559\pm 6.69\times 10^{-3}$  \\
Full precision &     $95.2\pm 0.0109$           &    $0.361\pm 5.92\times 10^{-5}$    \\ \hline
\end{tabular}
\label{tb:singlelayer}
\end{table}
\vspace{-3pt}

\begin{figure}[h]\footnotesize
\vspace{-0.3cm}
\begin{minipage}[H]{0.33\linewidth}
\centering
\includegraphics[width=1in]{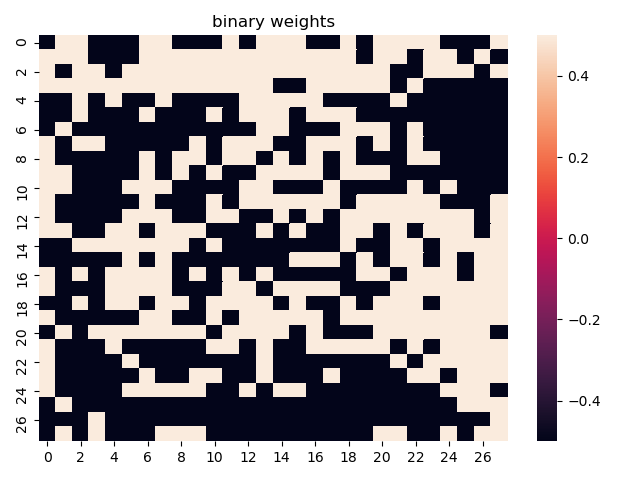}
\caption*{(a) 1 bit}
\end{minipage}%
\begin{minipage}[H]{0.33\linewidth}
\centering
\includegraphics[width=1in]{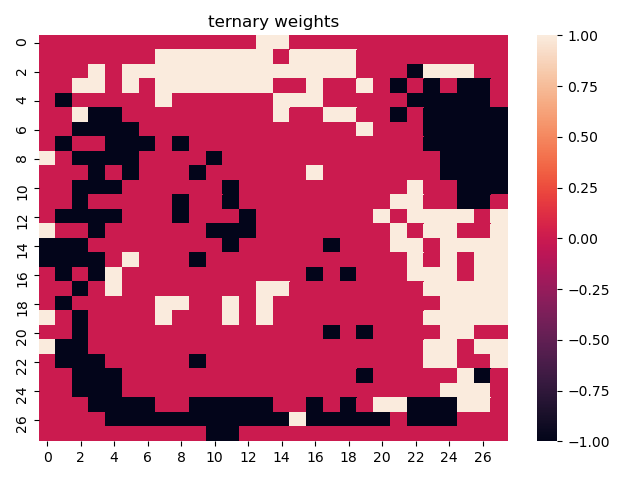}
\caption*{(b) 2 bits}
\end{minipage}
\begin{minipage}[H]{0.33\linewidth}
\centering
\includegraphics[width=1in]{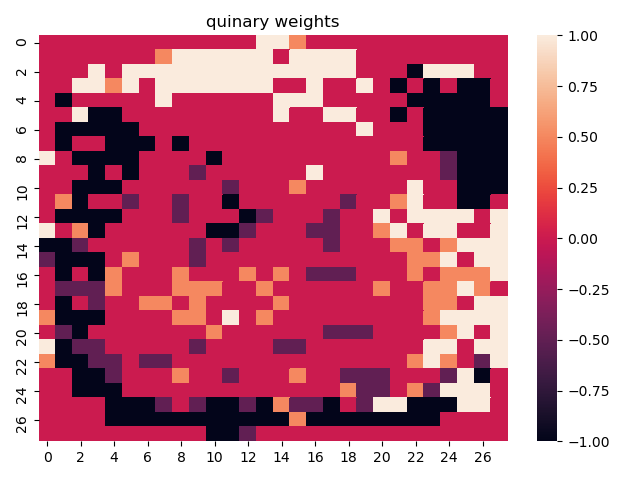}
\caption*{(c) 3 bits}
\end{minipage}
\\
\begin{minipage}[H]{0.32\linewidth}
\centering
\includegraphics[width=1in]{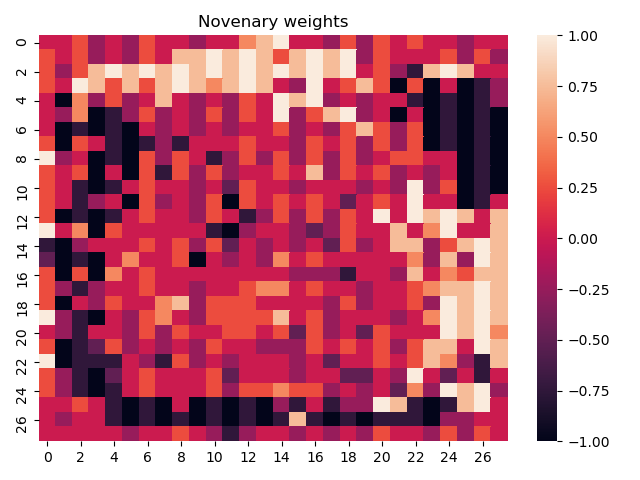}
\caption*{(d) 5 bits}
\end{minipage}
\begin{minipage}[H]{0.33\linewidth}
\centering
\includegraphics[width=1in]{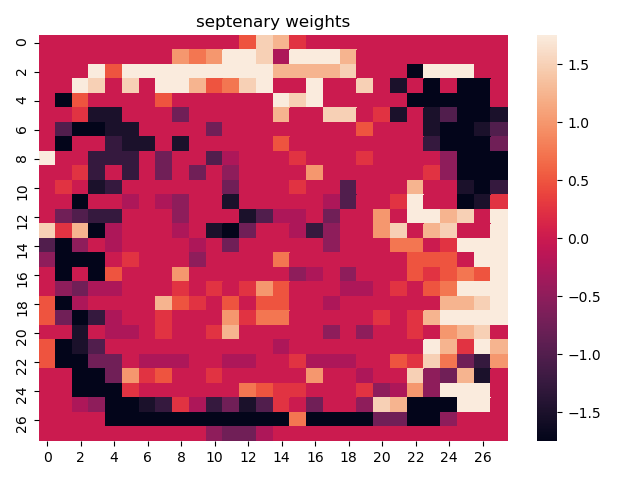}
\caption*{(e) 9 bits}
\end{minipage}
\begin{minipage}[H]{0.33\linewidth}
\centering
\includegraphics[width=1in]{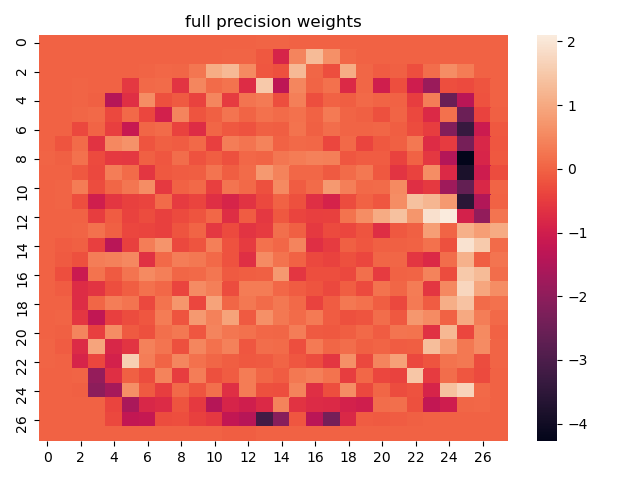}
\caption*{(f) Full precision}
\end{minipage}
\caption{Visualization of varying bit-width weight vectors for single layer network}
\label{fig:visulizeWeights}
\end{figure}

\noindent \textbf{[MLP on MNIST]} We apply our proposed methods on several networks, including a 2-layer fully connected network (FC2) with 100 hidden units, a 3-layer fully-connected network (FC3), a CNN6-d network~\cite{zhu2020r}, and a LeNet5~\cite{lecun1989backpropagation} with logistic loss as the binary classifiers to detect the first five digits among all the ten digits of the MNIST dataset. All these models are trained from scratch with weights that were initialized using the technique described in \cite{he2015delving}: $\beta=-\alpha=\sqrt{\frac{2}{n_l}}$. In addition, we implemented a hybrid learning strategy of these two approaches that uses GCD on the first half of the layers, and RSM on the second half.
\begin{table}\footnotesize
\caption{Comparisons of the performance on different model architectures using GCD, RSM, and a hybrid of them on MNIST dataset. }
\centering
\begin{tabular}{@{}c|c|c|c|c@{}}
\hline
                  & Model  & Acc.(\%) & Loss & Time(s)/Iter \\ \hline \hline
\multirow{4}{*}{\begin{tabular}[c]{@{}c@{}}Greedy \\ Coordinate \\ Descent\end{tabular}} & FC2    &    94.2 & 0.165      &     83         \\ & FC3    &  96.8       & 0.135     &  20728            \\ & CNN6   &    96.4      &0.100 & 8339              \\ & LeNet5 &   96.5       &0.0987      &     172         \\ \hline \hline
\multirow{2}{*}{Hybrid}                                                                  & CNN6   &    94.4      &0.145      &4525              \\ & LeNet5 &      96.0    &0.110      &102              \\ \hline \hline
\multirow{4}{*}{\begin{tabular}[c]{@{}c@{}}Supermodular \\ \\ Minimization\end{tabular}} & FC2    &    74.9      &  0.514    &       475       \\ & FC3    &   85.1       &0.373      &7298              \\ & CNN6   &   88.6       &0.289      & 1438             \\ & LeNet5 &   82.7       &0.384      & 258             \\ \hline
\end{tabular}
\label{tb:resultsMNIST}
\end{table}
\footnotetext[1]{The Acc. in the table refers to classification accuracy on the testing set.}
\footnotetext[2]{The time in the Tables was measured on Intel Core i7-7820X CPU with NVIDIA Corporation GP102 [GeForce GTX 1080 Ti] GPUs.}

The results in Table \ref{tb:resultsMNIST} shows that RSM fails to attain comparable accuracy with binary weights especially for small models i.e.\ FC2. But it can be more computationally efficient than GCD, especially for large models with millions of parameters i.e.\ CNN6 and FC3.  Comparing the results of hybrid training to those of GCD and RSM, we see that the lower performance of RSM optimization can be at least in part attributed to the increased non-convexity of earlier layers, meaning that the supermodular approximation less closely fits the true loss landscape.  However, for the second half of the layers, RSM performs on par with GCD, indicating a comparatively good fit with accompanying computational benefits. This hybrid approach combines the strengths of these two optimization approaches, and provides a trade-off between model performance and time consumption.\\
\noindent \textbf{[MLP on CIFAR10]} We also used the LeNet5 and CNN6 models for a binary classification task where the two classes are taken as the union of classes $0-2$ and $3-5$ from the CIFAR10 dataset, respectively. Between these two classes, we have 30,000 $32 \times 32 \times 3$ images for training and 6000 for testing. Using gradient descent, the CNN6 model with full-precision weights achieves an accuracy of $80.4\%$ with a loss $0.422$ on the testing set while the full-precision LeNet5 model attains an accuracy of $80.9\%$ with a loss $0.409$ on this task.

\begin{table}[H]\scriptsize
\caption{Testing performance using GCD, RSM with binary and ternary weights on CIFAR10 dataset.}
\begin{tabular}{l|c|c|c|c|c}
\hline
Model                   & Bit                      & Method & Acc.(\%) & Loss & Time(s)/Iter \\ \hline
\multirow{4}{*}{CNN6}   & \multirow{2}{*}{Binary}  & GCD          &   77.4   &   0.473   &   4216             \\
                        &                          & RSM       &   71.9       &   0.550   &  445             \\ \cline{2-6} 
                        & \multirow{2}{*}{Ternary} & GCD  &     78.7     &  0.456    &    6315               \\
                        &                          & RSM  & 72.5      &   0.540    &     477             \\ \hline
\multirow{4}{*}{LeNet5} & \multirow{2}{*}{Binary}  & GCD          &   74.5   &   0.526   &   114                    \\
                        &                          &  RSM       &   69.4       &   0.585   &  213           \\ \cline{2-6} 
                        & \multirow{2}{*}{Ternary} & GCD          &     74.9     &  0.502    &    165                \\
                        &                          &  RSM         &     69.9      &   0.577    &     386              \\ \hline
\end{tabular}
\label{tb:resultsCifar10}
\end{table}

\noindent The above results show that ternary weights slightly improve the model performance in terms of loss and accuracy compared with binary weights. And for deeper models i.e. CNN6, RSM provides a cheaper way to attain a comparable loss to GCD.
\vspace{-3pt}
\section{Discussion and Conclusions}
\vspace{-3pt}
 Our work has made the first step towards supermodular optimization on low bit-width neural networks and provided a framework to learn the network in a layerwise fashion with coordinate descent. For risk minimization, we build on non-monotone supermodular minimization to develop an algorithm with a global approximation guarantee (for each optimization step). We presented the randomized supermodular optimization and greedy coordinate descent for learning single and multi-layer neural networks, with extensive experiments on binary classification tasks. Greedy coordinate descent could reach comparable results to full precision but at the expense of high time complexity, especially for models with millions of parameters. However, supermodular optimization fails to attain comparable results in some cases. The hybrid of these two approaches has shown its excellent performance to attain a binary model from scratch in a resource- and time-efficient way. We also notice that coordinate descent is not an obvious candidate for massive parallelism, and it potentially relies on the initial state of the model parameters. Thus, if we could start from a well-initialized model, the performance of our approaches might be further improved. 
 
 The mapping from the general non-convex optimization landscape of deep neural network training to a convex layer-wise objective has its parallels in continuous optimization strategies \cite{pmlr-v97-belilovsky19a,NIPS2006_5da713a6}.  Convexification and layerwise optimization are themselves highly interesting and relevant areas of research that benefit both discrete and continuous optimization strategies.

\nocite{ex1,ex2}
\bibliographystyle{latex12}
\bibliography{reference.bib}

\newpage
\section*{Algorithms} \label{sec:supplements}
 
 \begin{algorithm}[H]\small
\SetKwFor{With}{with}{do}{}
\SetKwInOut{Input}{input}
\SetKwInOut{Output}{output}
\caption{\text{Greedy Coordinate Descent (GCD)}}
\Input{Dataset $D_n: \mathbb{R}^{d}\times \mathbb{R}$; weight vector $w^0 =\{\alpha,\beta\}^d$; objective function: $g(w)$}
\Output{weight vector $w^{d} \in \{\alpha,\beta\}^d$}
\BlankLine
\For{$i\leftarrow 1$ \KwTo $d$}{
\eIf{$w^{i-1}[i] = \alpha$ }{ 
      $w^{i}[i] = \beta$   }
      {$w^{i}[i] = \alpha$}
      \If{$g(w^{i-1})< g(w^{i})$}{$w^{i}\leftarrow w^{i-1}$ }
}
\label{algForwardGreedy}
\end{algorithm}

\begin{algorithm}[H]\small
\SetKwData{Left}{left}
\SetKwData{This}{this}
\SetKwData{Up}{up}
\SetKwFor{With}{with}{do}{endwith}
\SetKwFunction{FindCompress}{FindCompress}
\SetKwInOut{Input}{input}
\SetKwInOut{Output}{output}
\caption{Two-layer RSM ( or GCD)}
\KwIn{Training set: $D_n$; Weight of $2$nd layer: $a$; Iteration: $n\_iter$; Supermodular (Objective) function: $g(w)$}
\KwOut{$W =[w_1,\cdots,w_h]\in\{\alpha, \beta \}^{h\times d} $}
\For{$s = 1, \cdots, n\_iter$}{
\For{$i = 1, \cdots, h$}{
$w_i^{s} \leftarrow$ via RSM (or GCD) while keeping $w_j^{s}$ or $w_j^{s-1}$ ($j\neq i $) being fixed \\
\If{ $\operatorname{diff}= G(w_j^{s})- G(w_j^{s-1})>0$ \tcp*{ Accept-Reject} }{
\With{probability $1-\operatorname{sigmoid}(\frac{\operatorname{diff}}{0.05})$}{$w_j^{s}\leftarrow w_j^{s-1}$}
}
}
update $a^s\in\mathbb{R}^{h}$ via stochastic average gradient (SAG)
}
\label{algo_disjdecomp}
\end{algorithm}

\begin{algorithm}[h]\small
\vspace{-0cm}
\setlength{\abovecaptionskip}{0.05 cm}
\setlength{\belowcaptionskip}{0.05cm}
\SetKwFor{With}{with}{do}{endwith}
\caption{Multi-layer RSM (or GCD)}
\KwIn{Training set:$D_n$; Iteration: $n\_iter$}
\KwOut{$W_1,\cdots W_{L-1} $}
\BlankLine
\For{$s = 1, \cdots, n\_iter$}{
\For{$i = 1, \cdots, L-1$  }{
update $W_i^{s}$ via Algorithm~\ref{algo_disjdecomp}
}
}
\label{algoMultiLayerNetworks}
\end{algorithm}
\vspace{-0.2cm}

\end{document}